\documentclass[letterpaper, 10pt, conference]{ieeeconf}

\IEEEoverridecommandlockouts                          
\makeatletter

\let\proof\@undefined
\let\endproof\@undefined
\makeatother

\usepackage[switch]{lineno}

\usepackage[small]{caption}
\usepackage{subcaption}
\usepackage{footmisc}
\usepackage{url}
\usepackage{cite}

\usepackage{booktabs}
\usepackage{makecell}

\usepackage{mathtools, stackengine}

\usepackage{color}
\usepackage{xcolor}

\usepackage{graphicx} 

\usepackage{amsmath}
\usepackage{amsfonts}
\usepackage{amsthm}
\usepackage{amssymb}

\newtheorem{assumption*}{Assumption}
\newtheorem{theorem*}{Theorem}
\newtheorem{lemma}{Lemma}
\newtheorem{problem}{Problem}
\newtheorem{proposition}{Proposition}
\newtheorem{observation*}{Observation}
\newtheorem{remark*}{Remark}

\usepackage{algorithm}
\usepackage[noend]{algpseudocode}

\usepackage{enumerate}

\newcommand{\docupdate}[1]{{#1}}

\newcommand{\possibleremoval}[1]{{#1}}
\newcommand{\possibleaddition}[1]{}

\title{Approximate Environment Decompositions for Robot Coverage Planning using Submodular Set Cover}
\urldef{\ramesh}\url{m5ramesh@uwaterloo.ca}
\urldef{\smith}\url{stephen.smith@uwaterloo.ca}
\urldef{\fidan}\url{fidan@uwaterloo.ca}
\urldef{\imeson}\url{frank.imeson@avidbots.com}

\author{Megnath Ramesh, Frank Imeson, Baris Fidan, and Stephen L. Smith
\thanks{M. Ramesh (\ramesh) and S. L. Smith (\smith) are with the Department of Electrical and Computer Engineering and B. Fidan (\fidan) is with the Department of Mechanical and Mechatronics Engineering, at the University of Waterloo, Waterloo ON, Canada}
\thanks{F. Imeson (\imeson) is with Avidbots Corp., Kitchener ON, Canada}
}

\date{}

\begin{document}

\maketitle

\begin{abstract}
In this paper, we investigate the problem of decomposing 2D environments for robot coverage planning. Coverage path planning (CPP) involves computing a cost-minimizing path for a robot equipped with a coverage or sensing tool so that the tool visits all points in the environment. CPP is an NP-Hard problem, so existing approaches simplify the problem by decomposing the environment into the minimum number of \textit{sectors}. Sectors are sub-regions of the environment that can each be covered using a lawnmower path (i.e., along parallel straight-line paths) oriented at an angle. However, traditional methods either limit the coverage orientations to be axis-parallel (horizontal/vertical) or provide no guarantees on the number of sectors in the decomposition. We introduce an approach to decompose the environment into possibly overlapping \textit{rectangular} sectors. We provide an approximation guarantee on the number of sectors computed using our approach for a given environment. We do this by leveraging the submodular property of the sector coverage function, which enables us to formulate the decomposition problem as a submodular set cover (SSC) problem with well-known approximation guarantees for the greedy algorithm. Our approach improves upon existing coverage planning methods, as demonstrated through an evaluation using maps of complex real-world environments.


\end{abstract}

\section{Introduction}

Coverage path planning (CPP) is the problem of generating a minimum cost path for a robot equipped with a sensing or coverage tool such that the tool visits all points of the robot's environment \cite{galceranSurveyCoveragePath2013}. CPP has a wide range of applications in cleaning \cite{palacios-gasosOptimalPathPlanning2017}, agriculture \cite{jinCoverageControlAutonomous2013}, and surveillance tasks for search-and-rescue robots \cite{grontvedDecentralizedMultiUAVTrajectory2023}. CPP is an NP-Hard problem \cite{arkinApproximationAlgorithmsLawn2000} and thus research is primarily focused on computing approximations to the optimal coverage path. Planning approaches in literature typically use a two-step framework to make CPP more tractable, albeit at the cost of solution quality. This framework involves: (i) decomposing the environment into sub-regions which can each be individually covered, and (ii) computing a visitation order or a tour of the sub-regions. In this work, we refer to these sub-regions as "sectors". The environment decomposition step is critical as it determines the majority of the coverage path, and also affects the cost of the tour connecting the sectors. Therefore, it is important to compute decompositions that account for these effects on the coverage path.




In literature, CPP algorithms are grouped into exact or approximate approaches depending on the choice of environment decomposition. Exact approaches usually involve decomposing the environment into convex (or monotone) sectors \cite{bahnemannRevisitingBoustrophedonCoverage2021, huangOptimalLinesweepbasedDecompositions2001a, agarwalAreaCoverageMultiple2022a}. Each sector is then covered by a lawnmower path (i.e. using parallel coverage lines) along an optimal orientation as determined by the geometry of the sector. However, when applied to complex environments with irregular boundaries, exact CPP approaches are susceptible to \textit{over-decomposition} where some sectors are thinner than the coverage tool. This is a consequence of covering the entire environment and results in many instances of \textit{double coverage} where the robot returns to areas it has already covered. Additionally, these approaches provide few guarantees on the number of sectors in the decomposition.


In contrast, approximate CPP approaches use grid decompositions to cover areas of the environment that are reachable by the coverage tool. The decomposition is composed of non-overlapping (usually square) grid cells of size equal to the coverage tool \cite{luTMSTCPathPlanning2023, ianenkoCoveragePathPlanning2020, rameshOptimalPartitioningNonConvex2022, rameshAnytimeReplanningRobot2023}. The CPP problem is then to compute an optimal tour that visits each grid cell and minimizes the amount of double coverage. In our previous works \cite{rameshOptimalPartitioningNonConvex2022, rameshAnytimeReplanningRobot2023}, we proposed an approach that minimizes double coverage by minimizing the number of coverage lines (i.e. straight-line paths). However, using grid cells constrains the robot to cover the environment along axis-parallel (horizontal or vertical) orientations, leading to ``staircase-like" coverage paths which are sub-optimal. We look to achieve a decomposition that addresses the shortcomings of both exact and approximate approaches by (i) preventing over-decomposition and (ii) removing axis-parallel constraints. 



In this paper, we propose a decomposition of the environment into possibly overlapping rectangular sectors. We choose rectangles for the following reasons: (i) the lawnmower path to cover a rectangle can be optimally oriented along the longest edge \cite{huangOptimalLinesweepbasedDecompositions2001a}, and (ii) one can efficiently identify the largest rectangle in an environment \cite{marzehAlgorithmFindingLargest2019a, danielsFindingLargestArea1997}. We refer to this as the \textit{sector decomposition} problem. Inspired by works in sensor coverage \cite{corahDistributedSubmodularMaximization2018, mehrSubmodularApproachOptimal2018}, we leverage the \textit{submodularity} property of covering the environment with sectors to design our solution approach. Specifically, we show that the sector decomposition problem is an instance of the \textit{submodular set cover} (SSC) problem, which has been studied extensively in literature \cite{wolseyAnalysisGreedyAlgorithm1982, iyerSubmodularOptimizationSubmodular2013, NEURIPS2023_e5eaf67f}. In our proposed approach, we use the result that the greedy algorithm for SSC has a well-known approximation guarantee \cite{wolseyAnalysisGreedyAlgorithm1982}.


\noindent\textbf{Contributions:} Our specific contributions are as follows:
\begin{enumerate}
    \item We formulate the sector decomposition problem that aims to minimize the number of sectors and show that this is an instance of the submodular set cover problem.
    \item We propose a greedy algorithm to solve the sector decomposition problem, for which we provide an approximation guarantee on the number of sectors following that of \cite{wolseyAnalysisGreedyAlgorithm1982}.
    \item Finally, we present simulation results using maps of real-world environments and perform comparisons against state-of-the-art coverage planning approaches. 
\end{enumerate}


The paper is organized as follows. In Section \ref{sec:prelims}, we provide some background on the submodular set cover (SSC) problem. In Section \ref{sec:prob-def}, we introduce the sector decomposition problem. In Section \ref{sec:link-to_submod}, we motivate the minimization of sectors and show that the sector decomposition problem is an instance of SSC. In Section \ref{sec:sect_decomp}, we present our solution approach that greedily computes a sector decomposition for an environment. In Section \ref{sec:sect_touring}, we briefly describe how the decomposition is used to compute a coverage path for the environment. Finally, in Section \ref{sec:sim-results}, we provide the results of generating coverage paths for real-world environments using the proposed approach.

\section{Preliminaries}
\label{sec:prelims}

In this section, we provide a brief background on submodular functions and the submodular set cover (SSC) problem on which our decomposition approach is based.

\subsection{Submodular functions}

Let us consider a set $X$ (not necessarily finite) and let the set $2^X$ be the power set of $X$, i.e., the set of all subsets of $X$. A function $f : 2^X \rightarrow \mathbb{R}_{\geq 0}$ is \textit{submodular} if for all $A \subseteq B \subseteq X$ and $x \in X \setminus B$, we have
\[f(A \cup \{x\}) - f(A) \geq f(B \cup \{x\}) - f(B).\]
This is known as the property of diminishing returns. Equivalently, a function $f : 2^X \rightarrow \mathbb{R}_{\geq 0}$ is submodular if for every $A,B \subseteq X$, 
\[f(A) + f(B) \geq f(A \cup B) + f(A \cap B).\]

In addition to submodularity, we consider functions $f$ that possess the following additional properties:
\begin{enumerate}
    \item \textit{Monotonicity}: For all $A \subseteq B \subseteq X$, $f(A) \leq f(B)$,
    \item \textit{Normalization}: For the empty set $\emptyset$, $f(\emptyset) = 0$.
\end{enumerate}

\subsection{Submodular Set Cover}


Consider a set $X$ and a monotone submodular function $f: 2^X \rightarrow \mathbb{R}$ for which $f(X)$ is finite. The submodular set cover problem aims to compute a minimum cardinality subset $S \subseteq X$ subject to a submodular constraint \docupdate{$f(S) = f(X)$}. Formally, it can be represented by the following optimization problem:
\begin{align*}
    \min_{S \subseteq X} \quad& |S|\\
    \text{s.t.} \quad& \docupdate{f(S) = f(X)}.
\end{align*}

The submodular set cover problem is NP-Hard. However, a simple greedy algorithm provides the best-known performance guarantee for this problem \cite{wolseyAnalysisGreedyAlgorithm1982}. Let $\delta_{x}(S)$ for some $S \subseteq X$ and $x \in X \setminus S$ denote the marginal return obtained by adding $x$ to $S$: 
\[\delta_x(S) = f(S \cup \{x\}) - f(S).\] 

The greedy algorithm for this problem iteratively constructs a set $S$ as follows: at each iteration $r$, the algorithm adds the element $x \in X \setminus S$ that maximizes the marginal return $\delta_{x}(S)$. The greedy algorithm terminates after a feasible solution $S$ where $f(S) = f(X)$ is obtained.


\section{Problem definition}
\label{sec:prob-def}

Consider an indoor environment as represented by a set $\mathcal{W} \subseteq \mathbb{R}^2$. We look to cover the environment using the robot's coverage tool, which we assume to be a square of width $l$. This assumption follows from other works in coverage planning including \cite{arkinApproximationAlgorithmsLawn2000, rameshOptimalPartitioningNonConvex2022}. Specifically, we look to decompose $\mathcal{W}$ into (possibly overlapping) rectangular subsets called \textit{sectors}. Let $\mathcal{Q}$ denote a set of candidate sectors where each sector $Q_i \in \mathcal{Q}$ represents a rectangle in the environment (i.e., $Q_i \subseteq \mathcal{W}$) with the longest edge oriented along an angle $\theta_i \in [0,\pi)$. Note that the set $\mathcal{Q}$ can be uncountably infinite.


We now define the functions to compute the coverage of the environment using sectors. Let $|Q|$ denote the area of a set $Q \subseteq \mathbb{R}^2$, and let $a:2^\mathcal{Q} \rightarrow \mathbb{R}$ be a function that computes the area of the union of a subset $\mathcal{S} \subseteq \mathcal{Q}$:
\begin{align}
a(\mathcal{S}) = \Big\lvert\bigcup_{Q_i \in \mathcal{S}} Q_i\Big\rvert. \label{eq:area_fn}
\end{align}
We consider a set of candidate sectors $\mathcal{Q}$ such that its union covers $\mathcal{W}$, i.e. $a(\mathcal{Q}) = |\mathcal{W}|$ (i.e., total area of the environment). This is so that there exists a decomposition that covers the entire environment if necessary. 

We are now ready to present the main problem of this paper. Given the environment $\mathcal{W}$, we look to decompose $\mathcal{W}$ into the minimum number of rectangular sectors. However, we also look to prevent cases of over-decomposition, where the decomposition may contain small sectors covering inaccessible areas in the environment. To do this, we seek to find a decomposition such that the area within $\mathcal{W}$ covered by the sectors is at least $\gamma|\mathcal{W}|$, where $\gamma \in (0,1]$ is a design parameter we call the \emph{minimum coverage ratio}.

\begin{problem}[Sector Decomposition]
    Given an environment $\mathcal{W}$, a set of candidate rectangular sectors $\mathcal{Q}$ and a minimum coverage ratio $\gamma$, compute a set of sectors $\mathcal{S} \subseteq \mathcal{Q}$ that solves:
    \begin{align*}
        \min_{\mathcal{S} \subseteq \mathcal{Q}}  \quad& |\mathcal{S}|\\
        \text{s.t.} \quad& a(\mathcal{S}) \geq \gamma|\mathcal{W}|.
    \end{align*}
    
    \label{prob:sect_decomp}
\end{problem}

Using the sectors computed by solving Problem \ref{prob:sect_decomp}, we obtain the coverage lines (i.e., straight-line paths) that form the lawnmower path for the sector. Note that there are \textit{two} possible lawnmower paths depending on the connections between the coverage lines. From this, we generate a coverage path for the environment using an approach called \textit{sector touring} where: (i) we pick a lawnmower path for each sector, and (ii) connect the lawnmower paths by computing a cost-minimizing tour between the path endpoints. We refer to the connections between the lawnmower paths as \textit{transition paths}. An example coverage path for a set of three sectors is shown in Fig. \ref{fig:sector-coverage-approximation}. We provide more details on sector touring in Section \ref{sec:sect_touring}.

\begin{figure}
    \centering
    \includegraphics[width=\linewidth]{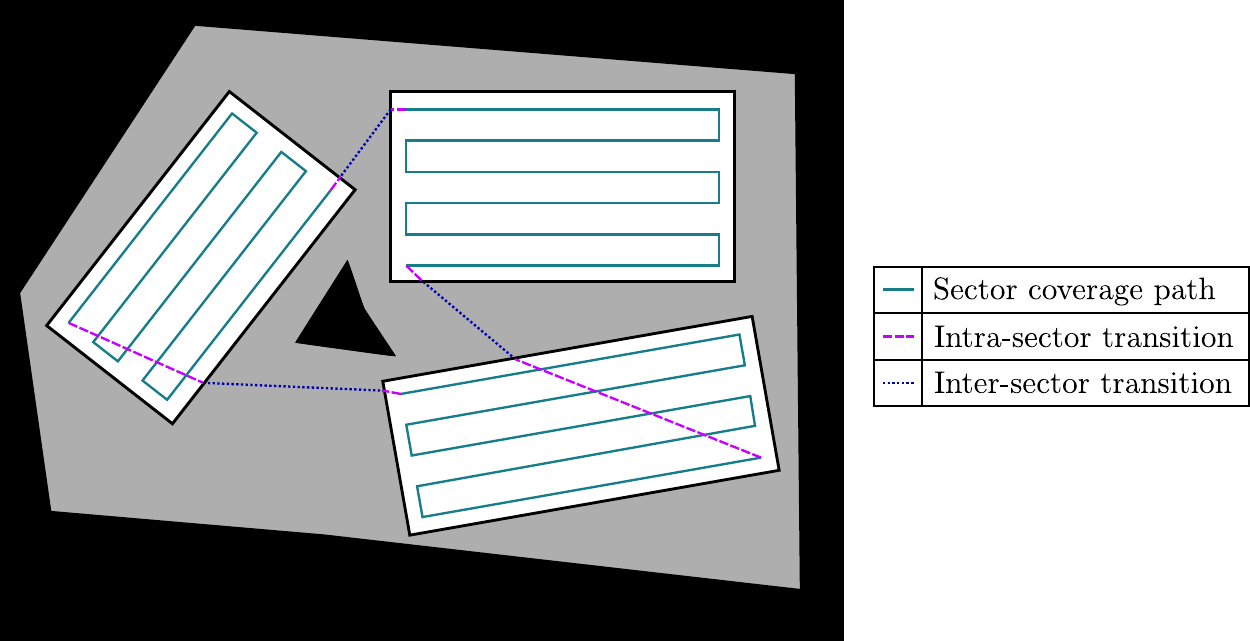}
    \caption{Coverage path resulting from an environment decomposition containing three sectors. Each sector is covered by a lawnmower path and the robot transitions between sectors using (i) intra-sector paths, and (ii) inter-sector paths.
    \vspace{-2mm}}
    \label{fig:sector-coverage-approximation}
\end{figure}

\begin{remark*}
    While we consider rectangular sectors in this paper, one could more generally consider sectors that are monotone with a well-defined orientation for the lawnmower path. However, in our solution approach, we leverage existing approaches \cite{marzehAlgorithmFindingLargest2019a} that efficiently compute the largest inscribed rectangle within the environment at an orientation.
\end{remark*}


\section{Link to Submodular Set Cover}
\label{sec:link-to_submod}

In this section, we establish that Problem \ref{prob:sect_decomp} is an instance of SSC. We motivate minimizing the number of sectors by deriving a bound on the worst-case coverage path length resulting from a decomposition of rectangular sectors. \possibleremoval{We then prove the submodularity of the area function $a$, which, based on~\cite{wolseyAnalysisGreedyAlgorithm1982}, implies that the greedy algorithm provides a polynomial-time approximation algorithm.} \possibleaddition{Following~\cite{wolseyAnalysisGreedyAlgorithm1982} and the submodularity of the area function $a$ (proof deferred to extended version due to space constraints), this implies that the greedy algorithm provides a polynomial-time approximation algorithm.}

\subsection{Coverage Path Optimality}
Here, we motivate the minimization of the number of sectors in the decomposition by considering the length of the coverage path resulting from sector touring. Giving an exact characterization of the relationship between the sector decomposition and the coverage path length is challenging. Thus, our approach is to bound the length of the coverage path obtained from sector touring and minimizing this bound. \docupdate{For convex environments, we may trivially bound the coverage path length using a single sector covering the environment. However, for non-convex environments, such as those considered in this paper, obtaining a similar bound requires more in-depth analysis.}

Consider a decomposition of the environment into rectangular sectors. The lawnmower path length to optimally cover a rectangular sector can be tightly bounded using the rectangle dimensions and the coverage tool width $l$. Consider a rectangular sector $Q$ of width $w$ and height $h$. Without loss of generality, let $w \geq h$ and consider a lawnmower path with coverage lines along the longest edge (i.e. width) of the rectangle. Now, let $P_Q$ be the path to cover $Q$ and let $L(P_Q)$ be its length.

\begin{lemma}
    Given a rectangular sector $Q$ of width $w$ and height $h$, the coverage path $P_Q$ has a length $L(P_Q)$ satisfying
    \[L(P_Q) \leq w \Big\lceil \frac{h}{l} \Big\rceil.\]
    
    \label{lem:rect_sector}
\end{lemma}

\possibleremoval{
\begin{proof}
    The length of each coverage line in the lawnmower path is $w - l$ as the tool can stay at a distance of $l/2$ away from the perimeter of the rectangle to cover it. To connect each coverage line to an adjacent line, we require a \textit{connecting path} of length at most $l$. So each line-connection pair has a length of $w$. 
    
    If $h$ is a multiple of $l$, then the number of line-connection pairs is exactly $h/l$, and the bound is tight. If $h$ is not a multiple of $l$, i.e., $h \neq cl$ for some $c \in \mathbb{N}$. The number of coverage lines needed to cover $Q$ is $\lceil h/l \rceil$. Therefore, the total path length is upper-bounded by $w \lceil h/l \rceil$, which proves the lemma.
\end{proof}}

\possibleaddition{Due to space constraints, we defer the proof of Lemma \ref{lem:rect_sector} to the extended version.} Let $\mathcal{S}$ be a set of rectangular sectors that cover the environment $\mathcal{W}$, i.e. $a(\mathcal{S}) = |\mathcal{W}|$. Given $\mathcal{S}$, let $P$ be the coverage path obtained by connecting the lawnmower paths using the aforementioned approach, and let $L(P)$ be the length of $P$. Since the sectors can overlap, the sum of sector areas is given by
\[\sum_{Q \in \mathcal{S}} |Q| = (1 + \alpha) |\mathcal{W}|,\]
where $\alpha$ is the ratio of $|\mathcal{W}|$ that is double-covered by sectors in $\mathcal{S}$. Also, let $w'$ represent the longest width of all sectors in $\mathcal{S}$.

Now, let $P^*$ be the optimal coverage path of $\mathcal{W}$ (unconstrained by sectors). We now bound $L(P)$ with respect to $L(P^*)$.

\begin{proposition}[Coverage Path Length Optimality]
    Given an environment $\mathcal{W}$ and a coverage tool of width $l$, let $P^*$ be an optimal coverage path of $\mathcal{W}$, with length $L(P^*)$. Consider a decomposition of $\mathcal{W}$ into rectangular sectors $\mathcal{S}$, and let $\alpha$ be the ratio of sector overlap and $w'$ be the longest sector width. Then, the coverage path $P$ resulting from sector touring has length $L(P)$ that satisfies:
    \begin{align*}
        L(P) &\leq (2 + \alpha) L(P^*) + (\sqrt{2} + 1)w'|\mathcal{S}|.
    \end{align*}
    
    \label{prop:cpp_length_opt}
\end{proposition}
\begin{proof}
    First, note that
    \begin{align}
        L(P^*) \geq \frac{|\mathcal{W}|}{l}. \label{eq:lower_bound}
    \end{align}
    This follows since the area covered by a tool of width $l$ along any path of length $L$ is at most $lL$, with equality if the path is a straight line.
    The path $P$ is obtained by (i) covering each sector individually using lawnmower paths, and (ii) connecting the lawnmower paths using transition paths. Let $L_s$ be the total length of the lawnmower paths. Following Lemma \ref{lem:rect_sector},
    \begin{align*}
        L_s \leq \sum_{Q \in \mathcal{S}} \Big( \frac{|Q|}{l} + w' \Big) \leq \frac{\sum_{Q \in \mathcal{S}} |Q|}{l} + |\mathcal{S}|w'
    \end{align*}

    Following the definition of $\alpha$, we get the following ratio between the sum of sector areas and $l$:
    \[\frac{\sum_{Q \in \mathcal{S}} |Q|}{l}= \frac{(1 + \alpha)|\mathcal{W}|}{l} \leq (1 + \alpha) L(P^*).\]
    
    We now bound the length of the transition paths $L_t$ that connect the lawnmower paths. We split the transition paths into two path segments: (i) inter-sector (path between sectors), and (ii) intra-sector (path traveled within each sector). Fig. \ref{fig:sector-coverage-approximation} illustrates these path segments in an example coverage path. To bound $L_t$, we consider computing transition paths as follows: (i) we use a TSP tour that visits a single corner of each sector (inter-sector), and (ii) connect each ends of the lawnmower path to the corner (intra-sector).
    
    To bound the inter-sector path segments, let TOUR($\mathcal{S}$) be the length of the TSP tour of the sector corners. Note that the optimal coverage path $P^*$ for the environment visits each sector at least once. So, we have that
    \[\mathrm{TOUR}(\mathcal{S}) \leq L(P^*).\]

    We now bound the intra-sector path segments for each sector. A lawnmower path for a rectangle starts close to one sector corner and ends at a different corner. Therefore, the cumulative path length to reach any corner from the endpoints of a lawnmower path is at most the diameter of the sector (length of the diagonal). Given that $w'$ is the longest sector width in the decomposition, this diameter is at most $\sqrt{2} w'$. As a result, we have that
    \[L_t \leq L(P^*) + \sqrt{2}|\mathcal{S}|w'.\]

    Adding all the path lengths, we arrive at the final result:
    \begin{align*}
    L(P) &= L_s + L_t \\
    &\leq (1 + \alpha) L(P^*) + (1 + \sqrt{2})|\mathcal{S}|w' + \mathrm{TOUR}(\mathcal{S}), \\
    &\leq (2 + \alpha) L(P^*) + (1 + \sqrt{2})|\mathcal{S}|w'
    \end{align*}
    which proves our proposition.
\end{proof}



Following Proposition \ref{prop:cpp_length_opt}, we observe two factors of the sector decomposition that affect the length of the worst-case coverage path: (i) the number of sectors $|\mathcal{S}|$, and (ii) the area of sector overlap. Therefore, reducing the number of sectors in the decomposition reduces the second term in the bound. Additionally, in our solution approach (Section \ref{sec:sect_decomp}), we address sector overlap by constraining the overlap area between sectors when computing the environment decomposition.

\subsection{Submodularity of sector coverage}

\possibleremoval{
We now show that the coverage function $a$ is a submodular function defined on the set of all candidate sectors $\mathcal{Q}$. \docupdate{This is not immediately obvious as $\mathcal{Q}$ is uncountably infinite.}}

\possibleaddition{\docupdate{Since the set of all candidate sectors $\mathcal{Q}$ is uncountably infinite, it is not immediately obvious that the coverage function $a$ is a submodular function defined on $\mathcal{Q}$.}}

\begin{proposition}
    The coverage function $a: 2^\mathcal{Q} \rightarrow \mathbb{R}$ as defined in Eq. (\ref{eq:area_fn}) is a normalized, monotone and submodular function.
    \label{prop:submodular}
\end{proposition}

\possibleremoval{
\begin{proof}
    Normalization holds trivially as the coverage of $\mathcal{W}$ is 0 if no sectors are picked from $\mathcal{Q}$. Similarly, monotonicity holds trivially, since for any set $A \subseteq \mathcal{Q}$, each sector in $A$ must cover a non-negative amount of area in $\mathcal{W}$. 
    
    To show the submodularity of $a$, consider sets $A \subseteq B \subset \mathcal{Q}$ and a sector $ Q \in \mathcal{Q}\setminus B$. From the definition of submodularity in Section \ref{sec:prelims}, the following must hold:
    \begin{align*}
        a(A \cup \{Q\}) - a(A) \geq a(B \cup \{Q\}) - a(B).
    \end{align*}
    For the set $B$, let $\delta_Q(B) = a(B \cup \{Q\}) - a(B)$ represent the marginal return of adding $Q$ to the set $B$. First, notice that
    \[
    \delta_Q(B) = \Big\lvert Q \setminus (\cup_{Q_B \in B} Q_B) \Big\rvert.
    \]
    Now, manipulating this expression we get
    \begin{align*}
        \delta_Q(B) 
        &= \Big\lvert Q \setminus (\cup_{Q_B \in B} Q_B) \Big\rvert\\ 
        &=  \Big\lvert Q \setminus \Big((\cup_{Q_A \in A} Q_A) \cup (\cup_{Q_B \in B\setminus A} Q_B)\Big)\Big\rvert \\ 
        &= \Big\lvert \Big(Q \setminus (\cup_{Q_A \in A} Q_A)\Big) \cap \Big(Q \setminus (\cup_{Q_B \in B\setminus A} Q_B)\Big)\Big\rvert \\
        &\leq \Big\lvert Q \setminus (\cup_{Q_A \in A} Q_A) \Big\rvert\\ &= \delta_Q(A),
    \end{align*}
where the equality in the third line follows from an application of DeMorgan's Laws (i.e., $A \setminus (B\cup C) = (A \setminus B) \cap (A \setminus C)$). This proves the proposition.
\end{proof}}

\possibleaddition{We provide a proof for Proposition \ref{prop:submodular} in the extended version.} Following Proposition \ref{prop:submodular}, we have that Problem \ref{prob:sect_decomp} is an instance of SSC which allows us to use the result from \cite{wolseyAnalysisGreedyAlgorithm1982} to propose a greedy decomposition approach.

    

\section{Sector Decomposition Approach}
\label{sec:sect_decomp}

In this section, we describe our solution approach for the sector decomposition problem. Our approach greedily adds large rectangular sectors to the decomposition that cover new areas of environment. We first describe an approach that leverages computational geometry algorithms to identify large rectangular sectors in the environment (or its subset). We then present a greedy algorithm that computes a sector decomposition while constraining the overlap between the sectors. For the case where the overlap is not constrained, we provide an approximation factor for the greedy algorithm following \cite{wolseyAnalysisGreedyAlgorithm1982}. 

\subsection{Sector identification}
\label{sec:sect-id}

We first identify the largest candidate sector for a given environment $\mathcal{W}$ (or its subset). For this, we solve the problem of computing a rectangle of maximum area in $\mathcal{W}$. In computational geometry literature, this is known as the convex skull or the "potato peeling" problem. Computing a convex skull of general shape and orientation requires $O(n^7)$ runtime, where $n$ is the number of vertices in the polygonal representation of $\mathcal{W}$\cite{changPolynomialSolutionPotatopeeling1986}. Another alternative is to use approximation algorithms \cite{hall-holtFindingLargeSticks2006a} which reduce runtime at the sacrifice of optimality. However, one can compute the largest area axis-parallel rectangle in the environment efficiently using existing algorithms \cite{marzehAlgorithmFindingLargest2019a, danielsFindingLargestArea1997}. Additionally, for the case of coverage planning, one can use the orientations of the environment's edges (both boundary and hole edges) to constrain the robot's coverage directions \cite{huangOptimalLinesweepbasedDecompositions2001a, bahnemannRevisitingBoustrophedonCoverage2021, agarwalAreaCoverageMultiple2022a}. This is to ensure that the robot driving the coverage path covers along long walls and moves predictably in the environment.

Our approach to identify the largest sector is as follows. First, we use a polygonal representation of $\mathcal{W}$ and extract the orientations of the edges to construct a set of candidate orientations $\Theta$. We limit our set of candidate sectors $\mathcal{Q}$ to rectangles oriented along an angle in $\theta \in \Theta$. Given this constraint, we define an axis oriented along each $\theta$ and obtain the largest axis-parallel rectangular sector in $\mathcal{W}$ for each angle. Repeating this for every candidate orientation gives us a set of candidate sectors $\Bar{\mathcal{Q}}$ from which we choose the largest one in our greedy approach.

\subsection{Greedy decomposition}

We now propose an algorithm that decomposes the environment $\mathcal{W}$ by greedily choosing rectangular sectors that cover the largest remaining uncovered area of $\mathcal{W}$. The pseudocode for our approach is given in Algorithm \ref{alg:greedy-sector-decomp}. Until the minimum coverage of the environment is achieved (Line \ref{line:env_uncovered}), we compute a candidate rectangular sector for each candidate orientation $\theta \in \Theta$ (Line \ref{line:candidate_secs}). We then pick the candidate sector with the largest area and add it to our decomposition $\mathcal{S}$ (Lines \ref{line:max}-\ref{line:add}). Once the sector is added, we remove the area covered by the sector from the environment so that future sectors may not have overlapping coverage (Line \ref{line:remove}).

However, in some cases, a small amount of overlap between sectors is useful to cover more area and have better fitting neighbouring sectors. To do this, we introduce a parameter called sector erosion radius $\beta$. This radius is used to compute a Minkowski difference between $Q$ and a disk of radius $\beta$. The difference $\Bar{Q}$ is removed from the environment so that future sectors may overlap with $Q$. Setting $\beta$ to $0$ leads to truly non-overlapping sectors, while arbitrarily large values lead to unconstrained overlapping. In our simulation results in Section \ref{sec:sim-results}, we further motivate the allowance of a small amount of overlap. 

\begin{algorithm}
\caption{Greedy Sector Decomposition (G-Sect)}
\label{alg:greedy-sector-decomp}
 \begin{algorithmic}[1]
 \renewcommand{\algorithmicrequire}{\textbf{Input:}}
 \renewcommand{\algorithmicensure}{\textbf{Output:}}
 \Require Environment $\mathcal{W}$, Candidate orientations $\Theta$, Minimum coverage ratio $\gamma$, Sector erosion radius $\beta$
 \Ensure Sector decomposition $\mathcal{S}$
 \State $\mathcal{S} \gets \emptyset$ \Comment{Intialize to empty set}
 \State $\Bar{\mathcal{W}} \gets \mathcal{W}$ \Comment{Copy environment}
 \While{$a(\mathcal{S}) < \gamma|\mathcal{W}|$} \label{line:env_uncovered}
 \State $\Bar{\mathcal{Q}} \gets $ Candidate sectors for $\Bar{\mathcal{W}}$ for all angles in $\{\theta\}$ \label{line:candidate_secs}
 \State $Q \gets \arg\max_{Q \in \Bar{\mathcal{Q}}} a(\mathcal{S} \cup \{Q\}) - a(\mathcal{S})$ \label{line:max}
 \State $\mathcal{S} \gets \mathcal{S} \cup \{Q\}$ \label{line:add}
 \State $\Bar{Q} \gets$ Erode $Q$ by a radius of $\beta$ \label{line:erode}
 \State $\Bar{\mathcal{W}} \gets \Bar{\mathcal{W}} \setminus \Bar{Q}$ \label{line:remove}
 \EndWhile
 \Return $\mathcal{S}$
\end{algorithmic}
\end{algorithm}

\begin{remark*}
For the case of unconstrained overlap, the candidate sectors must be computed using an alternate approach, as the convex skull method may return the same rectangle repeatedly. One approach is to use a checkerboard decomposition of the environment \cite{vandermeulenTurnminimizingMultirobotCoverage2019} to identify rectangular sub-regions at each orientation. However, the candidate sectors only need to be computed once for this case.
\end{remark*}

\subsection{Analysis}

We now provide a brief analysis of the G-Sect algorithm. In the case of overlapping sectors (large $\beta$), we obtain an approximation factor for the G-Sect algorithm as a function of the minimum coverage ratio $\gamma$. We derive this factor following that of Wolsey's analysis of the greedy algorithm in \cite{wolseyAnalysisGreedyAlgorithm1982}. Given a set of candidate orientations $\Theta$, let $\mathcal{S}^*$ be the minimum sector decomposition of the environment, i.e. $|\mathcal{S}^*| \leq |\mathcal{S}|$ for any $\mathcal{S}$.

\begin{proposition}[Approximation factor for G-SectD]
    Given an environment $\mathcal{W}$ and a minimum coverage ratio of $\gamma$, let $\mathcal{S}^*$ be the minimum sector decomposition given candidate orientations $\Theta$. Then, the G-Sect algorithm computes a sector decomposition $\mathcal{S}$ satisfying
    \begin{align*}
        \frac{|\mathcal{S}|}{|\mathcal{S^*}|} \leq 1 + \ln
    \bigg( \frac{1}{1 - \gamma} \bigg).
    \end{align*}
    
    \label{prop:approx-factor}
\end{proposition}

\possibleremoval{\begin{proof}
    We first prove this proposition for a \textit{finite} set of candidate sectors $\mathcal{Q}$ that cover the environment $\mathcal{W}$. Let $T$ be the number of iterations run by the greedy algorithm, and let $\mathcal{S}^r$ for $1 \leq r \leq T$ be the subset that is picked by the greedy algorithm after $r$ iterations. Following Wolsey's analysis of the greedy algorithm for SSC over a finite set \cite{wolseyAnalysisGreedyAlgorithm1982}, we have the following approximation factor:
    \begin{align}
        \frac{|\mathcal{S}|}{|\mathcal{S^*}|} &\leq 1 + \ln
    \bigg( \frac{a(\mathcal{Q}) - a(\emptyset)}{a(\mathcal{Q}) - a(\mathcal{S}^{T-1})} \bigg). \label{eq:wolsey_bound}
    \end{align}
    Since $a$ is normalized, we have $a(\emptyset) = 0$. Since $\mathcal{Q}$ covers $\mathcal{W}$, we have that $a(\mathcal{Q}) = |\mathcal{W}|$. Given that G-Sect terminates after a minimum coverage of $\gamma|\mathcal{W}|$, we have that $a(\mathcal{S}^{T-1}) \leq \gamma|\mathcal{W}|$. Substituting these values in Eq. (\ref{eq:wolsey_bound}), we get
    \begin{align}
        \frac{|\mathcal{S}|}{|\mathcal{S^*}|} &\leq 1 + \ln
    \bigg( \frac{|\mathcal{W}|}{|\mathcal{W}| - \gamma |\mathcal{W}|} \bigg) = 1 +  \ln
    \bigg( \frac{1}{1 - \gamma} \bigg)
    \label{eq:our-bound}
    \end{align}
    
    We now show that this bound also applies when $\mathcal{Q}$ is infinite. Let $\mathcal{S}^* = \{Q^*_1, Q^*_2, \dots Q^*_k\}$ be the minimum sector decomposition of $\mathcal{W}$ from all of $\mathcal{Q}$. Now, let $\mathcal{S} = \{Q_1, Q_2, \dots Q_T\}$ be the sector decomposition obtained by G-Sect given $\gamma$. Using these sets, we construct a finite set of candidate sectors $\mathcal{Q}' = \mathcal{S} \cup \mathcal{S}^*$. On this finite set $\mathcal{Q}'$, we can run G-Sect and for each greedy iteration $r$, we have the following for all $Q_i^* \in \mathcal{S}^*$:
    \begin{align*}
        a(\mathcal{S}^r \cup Q_{r+1}) - a(\mathcal{S}^r) \geq a(\mathcal{S}^r \cup Q^*_{i}) - a(\mathcal{S}^r).
    \end{align*}  
    This follows from the correctness of the greedy step for a set of candidate orientations $\Theta$. As a result, the approximation factor holds on the finite set $\mathcal{Q}'$. However, since $\mathcal{Q}'$ contains the choices of G-Sect from $\mathcal{Q}$ and the optimal set from $\mathcal{Q}$, the approximation factor also holds on the infinite set $\mathcal{Q}$. This proves the proposition.
\end{proof}}


\possibleaddition{The proof of Proposition \ref{prop:approx-factor} is deferred to the extended version due to space constraints.} Obtaining a similar approximation factor for the case of minimally-overlapping sectors is a non-trivial task. However, we have found that decompositions using minimally overlapping sectors are better suited for coverage planning as they reduce instances of double coverage. In Section \ref{sec:sim-results}, we provide simulation results of planning coverage paths using minimally overlapping sectors.

\subsection{Local sector merges}

\possibleaddition{After G-Sect computes a decomposition, we merge neighbouring sectors to further reduce the number of sectors. Specifically, we conduct merges where the lawnmower path of one sector can be extended to cover a neighbouring sector. Note that these merges may result in sectors that are not necessarily rectangular, i.e. they do not strictly satisfy Problem \ref{prob:sect_decomp}. However, they still satisfy the key restriction that each sector is covered by a lawnmower path. More details about the merging is provided in the extended version.}

\possibleremoval{
After a decomposition is computed using G-Sect, we merge neighbouring sectors to further reduce the number of sectors. Specifically, we conduct merges where the lawnmower path of one sector can be extended to cover a neighbouring sector. Given the sector decomposition and corresponding sector orientations, we start by counting the number of coverage lines required to cover each sector. We then iterate through the sectors in the order of increasing sector area (i.e. smallest to largest) to check for possible merges. For each sector $Q$, we look for an adjacent sector $Q_{\text{adj}}$ with a coverage orientation $\theta_{\text{adj}}$ such that the number of coverage lines to cover $Q \cup Q_{\text{adj}}$ along the orientation $\theta_{\text{adj}}$ is less than or equal to the cumulative number of lines required to cover each sector individually. The reduction in coverage lines indicates that the lawnmower path of $Q_{\text{adj}}$ can be extended to cover $Q$. Any ties are broken by choosing $Q_{\text{adj}}$ with the largest area. The merged sector is added back to the set of sectors and the process is repeated until all possible merges have been completed. The resulting sector decomposition forms our final solution to Problem \ref{prob:sect_decomp}.

Note that these merges may result in sectors that are not necessarily rectangular, i.e. they do not strictly satisfy Problem \ref{prob:sect_decomp}. However, they still satisfy the key restriction that each sector is covered by a lawnmower path. To obtain a strictly rectangular decomposition, this step may be skipped. In practice, a similar strategy may also be used to cover areas that are uncovered by the sector decomposition.
}
\section{Sector Touring}
\label{sec:sect_touring}

\begin{table}
\vspace{6pt}
\centering
\caption{Robot parameters used for experiments}
\label{table:parameters}
\begin{tabular}{@{} l l@{}}
\toprule
Parameter & Value \\
\midrule
Coverage tool width $l$ & 0.8 $m$\\
Maximum linear velocity & 1 $m/s$\\
Linear acceleration & $\pm 0.5$ $m/s^2$ \\
\bottomrule
\end{tabular}
\vspace{-2mm}
\end{table}

We now describe our approach to obtain a coverage path for the environment using the sector decomposition from Section \ref{sec:sect_decomp}. Similar to \cite{bahnemannRevisitingBoustrophedonCoverage2021}, we formulate this as a GTSP that \docupdate{aims to minimize the time to cover all sectors by determining} \docupdate{(i) }the visitation order of the sectors in the decomposition\docupdate{, and (ii) entry and exit points for each sector.} We create an auxiliary graph where the vertices are grouped into sets. Each set in the graph represents the coverage of a sector and consists of at most four vertices that represent the directed paths to cover the sector, i.e. two possible lawnmower paths per sector where each path can be traversed in two directions. \docupdate{The chosen path and traversal direction determine the entry and exit points for each sector.} The graph edges represent the transition paths to traverse from one sector to another. The edge costs are determined by the time to travel along an obstacle-free path from one end of a lawnmower path to another determined using a visibility graph planner \cite{obermeyerVisiLibityLibraryVisibility2008}. We determine this travel time using a piecewise constant acceleration robot model where the robot travels along a straight line with constant acceleration until a maximum velocity is reached. \docupdate{An example set of parameters for this model is given in Table \ref{table:parameters}.} In this work, we do not directly penalize the turning time of the robot, but the above robot model penalizes the time to stop to perform an upcoming turn. The GTSP tour on the resulting graph gives the coverage path for the environment as a series of connected sectors.

\section{Results}
\label{sec:sim-results}

In this section, we present simulation results of coverage planning using the proposed sector decomposition approach for maps of real-world environments. Specifically, our dataset includes 2 street maps of New York and Shanghai cities from \cite{luTMSTCPathPlanning2023} and 2 anonymized environments obtained from combining maps of real-world environments obtained from our industry partner Avidbots. We compare the coverage performance of our approach against two state-of-the-art approaches: (i) an exact decomposition approach called boustrophedon cell decomposition (BCD) \cite{bahnemannRevisitingBoustrophedonCoverage2021}, and (ii) our previous approximate decomposition approach called Optimal Axis-Parallel Rank Partitioning (OARP) approach \cite{rameshOptimalPartitioningNonConvex2022}.


For our simulations, we set the desired coverage of the decomposition $\gamma$ to be $95\%$ of the environment, after which the greedy algorithm is terminated. \docupdate{This choice helps eliminate sectors that cover small and inaccessible areas and effectively reduces the number of sectors in the decomposition.} We identify the environment's walls using Probabilistic Hough Transform \cite{kalviainenProbabilisticNonprobabilisticHough1995} and obtain the set of candidate coverage orientations $\Theta$. Table \ref{table:parameters} shows the robot parameters used for the simulations, including the coverage tool width $l$. For the BCD and sector decomposition approaches, we ignore small obstacles (obstacle area $< 4l^2$) as they can be evaded effectively during coverage and do not affect the coverage orientation. For OARP, some these obstacles are removed automatically during the construction of the grid decomposition.

To motivate our configuration of G-Sect pertaining to sector overlap, we present the results of computing a sector decomposition for a simpler test map using two values of the sector erosion radius $\beta$: $l/4$ and $2l$. Figure \ref{fig:angled comparison} shows the resulting decompositions, where we observe that increasing the amount of overlap reduces the number of sectors but also results in overlapping coverage lines (i.e., double coverage). However, we observe that a small amount of overlap with local sector merges identifies a decomposition with non-overlapping coverage lines. Therefore, we conduct the rest of the experiments with $\beta = l/4$ and conduct sector merges.

\begin{table}
    \centering
    \caption{Comparison of number of sectors in each decomposition}
    \begin{tabular}{c|c|c}
        \toprule
         Map &  BCD \cite{bahnemannRevisitingBoustrophedonCoverage2021} & \textbf{Our Approach}  \\
         \midrule
         Anonymous 1 & 55 & \textbf{32} \\
         Anonymous 2 & 75 & \textbf{43} \\
         New York & 136 & \textbf{73} \\
         Shanghai & 86 & \textbf{44} \\
         \bottomrule
    \end{tabular}
    \label{tab:decomposition}
\end{table}

\begin{figure}
    \centering
    \subcaptionbox{$\beta = l/4$}
    {
    \includegraphics[width=0.295\linewidth]{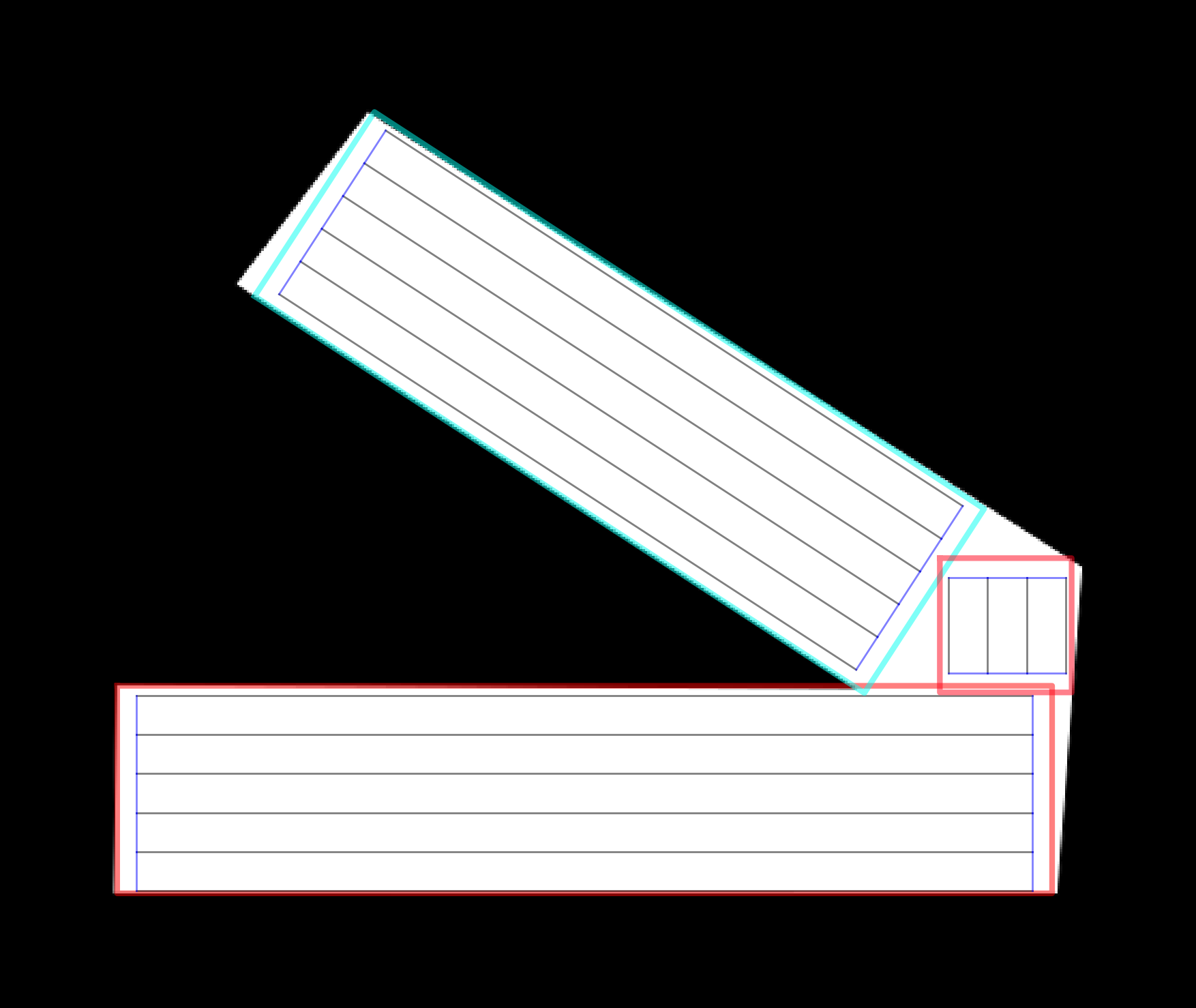}
    }
    \vspace{2mm}
    \subcaptionbox{$\beta = 2l$}
    {
    \includegraphics[width=0.295\linewidth]{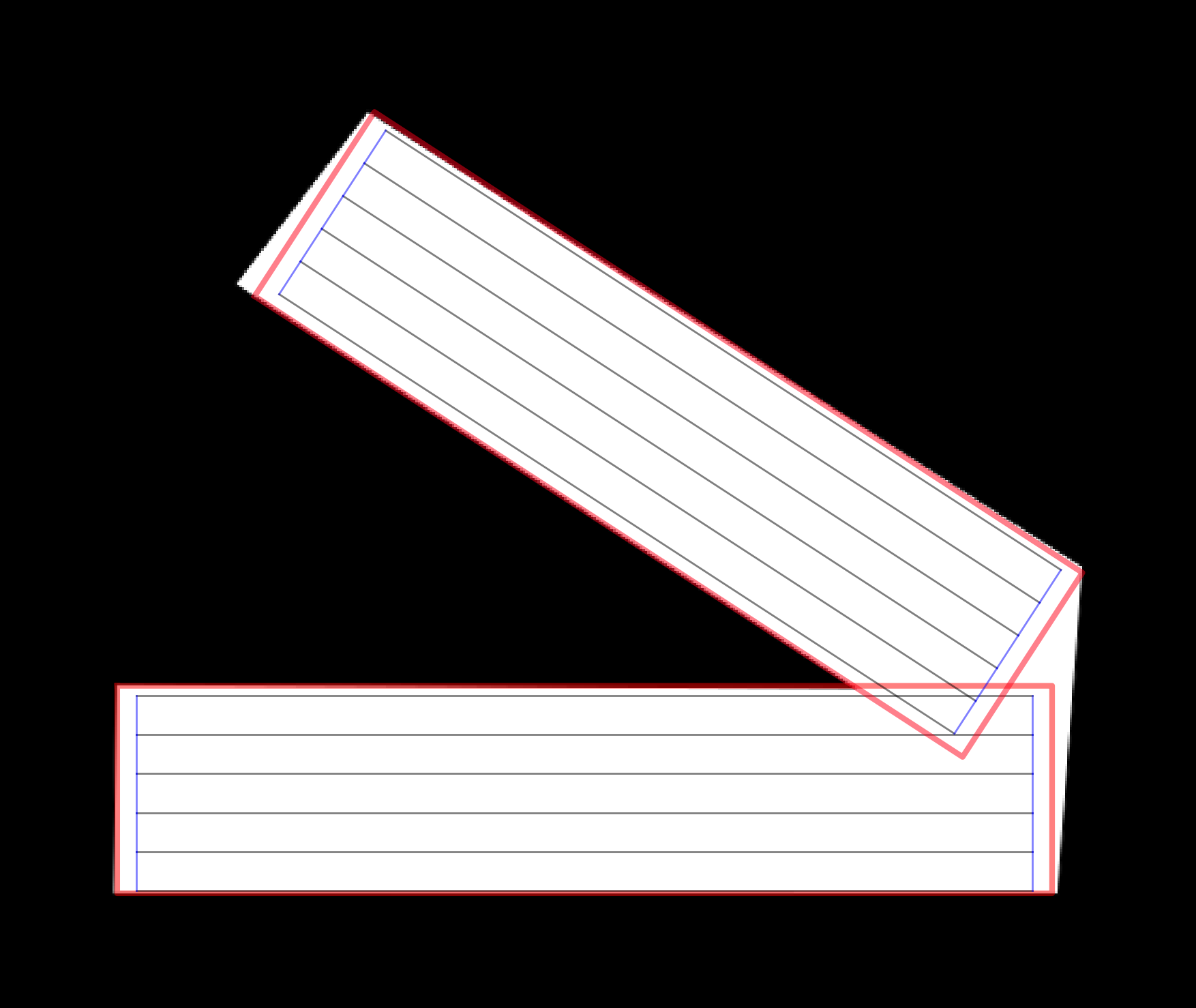}
    }
    \subcaptionbox{$\beta = l/4$ \& merging}
    {
    \includegraphics[width=0.295\linewidth]{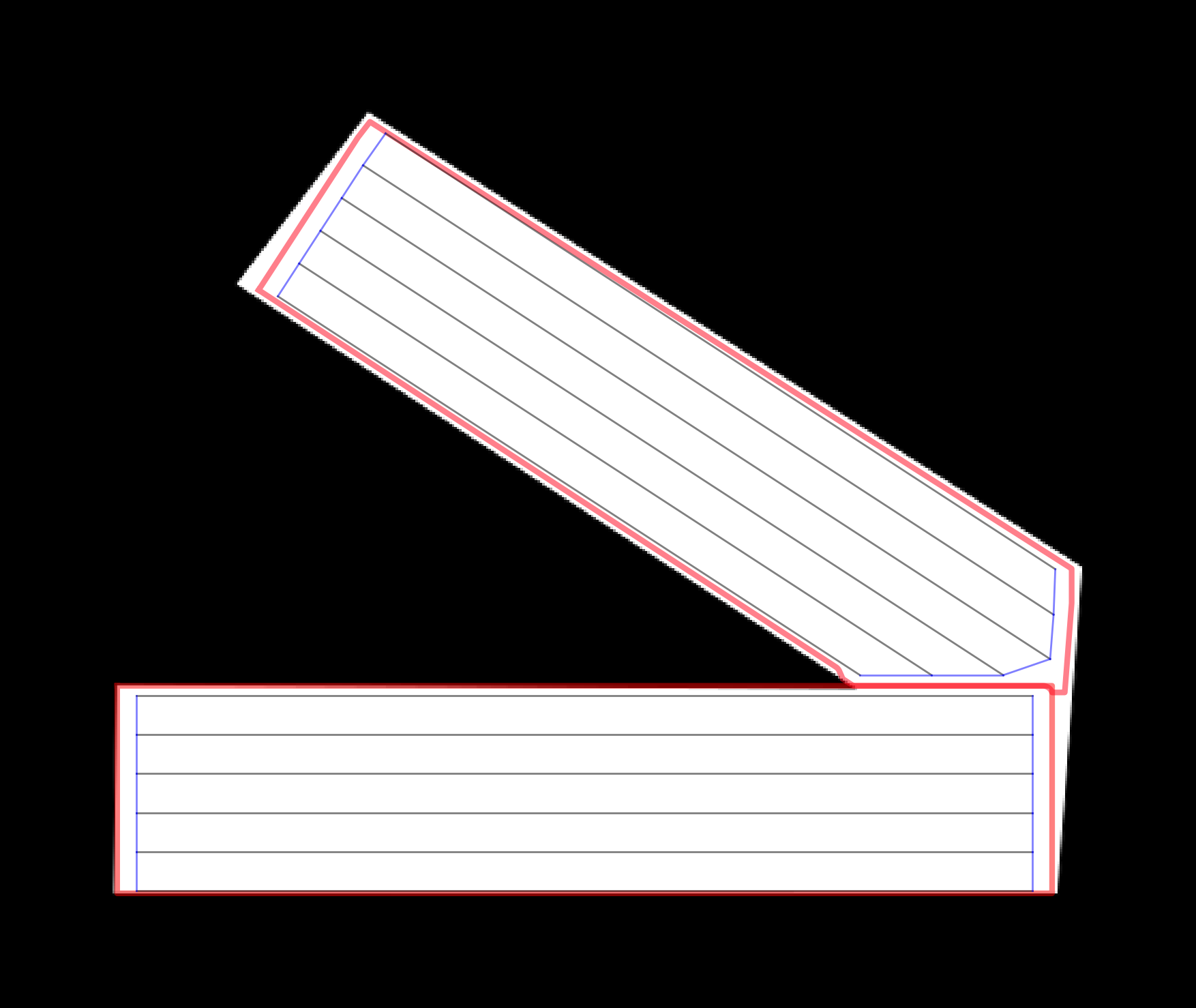}
    }
    \caption{G-Sect results with varying sector overlap and merging.
    \vspace{-4mm}}
    \label{fig:angled comparison}
\end{figure}

Figure \ref{fig:qualitative-results} shows the decompositions computed using G-Sect for our test maps. We observe that sector decomposition results in intuitive coverage lines, especially in narrow regions where the coverage lines are oriented along the edges of the environment. Table \ref{tab:decomposition} shows the results of comparing sector decomposition against that of BCD (exact decomposition) \cite{bahnemannRevisitingBoustrophedonCoverage2021}. \docupdate{We observe that sector decomposition approach covers the environment with fewer sectors, predominantly due to (i) our choice of $\gamma$, and (ii) sector merging.} In contrast, BCD results in over-decomposition due to the large number of reflex vertices present in the test environments. G-Sect is less influenced by these vertices for the chosen $\gamma$.


We now compare the coverage planning performance of the sector decomposition approach against that of BCD and OARP. The robot model from Section \ref{sec:sect_touring} is used to compute the coverage path cost with parameters as per Table \ref{table:parameters}. Table \ref{table:performance} shows the results of comparing (i) number of coverage lines, (ii) percent area covered, and (iii) coverage path cost. For every map, we run each approach for 5 trials and present the average for each metric. We observe that the proposed approach achieves the best coverage path cost for all test maps. We believe that this is due to the intuitive coverage line placements achieved by G-Sect, e.g., the coverage lines align with the boundaries and the holes. For the New York and Shanghai maps, we observe that OARP computes paths with higher cost than the proposed approach despite having fewer coverage lines. This is because constraining the robot's motion to axis-parallel (horizontal/vertical) orientations creates longer transition paths (where the robot performs no coverage) that add to the cost. However, for environments with axis-parallel boundaries, we expect the OARP approach to perform better because minimizing the number of coverage lines results in fewer transitions. Thus, a future direction we look to explore is a combination of sector decomposition and OARP \cite{rameshOptimalPartitioningNonConvex2022} that minimizes the number of coverage lines in the path while considering multiple coverage orientations. Meanwhile, the BCD approach covers the environment entirely, but results in relatively expensive coverage paths as a result of over-decomposition.

\begin{table*}
	\centering
 \caption{Comparison of coverage planning performance for the test maps.}
	\begin{tabular}{@{} l rrr l rrr l rrr l@{}}
		\toprule
		& \multicolumn{3}{c}{BCD \cite{bahnemannRevisitingBoustrophedonCoverage2021}} & \phantom{a} & \multicolumn{3}{c}{OARP \cite{rameshOptimalPartitioningNonConvex2022}} & \phantom{a} & \multicolumn{3}{c}{\textbf{Sector Decomposition}}\\
		\cmidrule{2-4} 		\cmidrule{6-8}  \cmidrule{10-12} 
		& \# Lines & $\%$ Area & Cost (s) & & \# Lines & $\%$ Area & Cost (s) & & \# Lines & $\%$ Area & Cost (s) \\
		\midrule
  Anonymous 1 & 187 & \textbf{100\%} & 2297 & & 169 & 95.4\% & 1850 & & \textbf{158} & 95\% & \textbf{1714} \\
		Anonymous 2 & 204 & \textbf{100\%} & 2566 & & 130 & 94.2\% & 2425 & & \textbf{125} & 95\% & \textbf{2255}\\
		New York & 300 & \textbf{100\%} & 4692 & & \textbf{219} & 94.3\% & 3880 & & 267 & 95.1\% & \textbf{3874}\\
		Shanghai & 235  & \textbf{100\%} & 4540 & & \textbf{187} & 97\% & 3713 & & 221 & 95\% & \textbf{3687} \\

		\bottomrule

	\end{tabular}
	\label{table:performance}
\end{table*}

\begin{figure*}
    \centering
    \subcaptionbox{New York}
    {
    \includegraphics[width=0.17\linewidth]{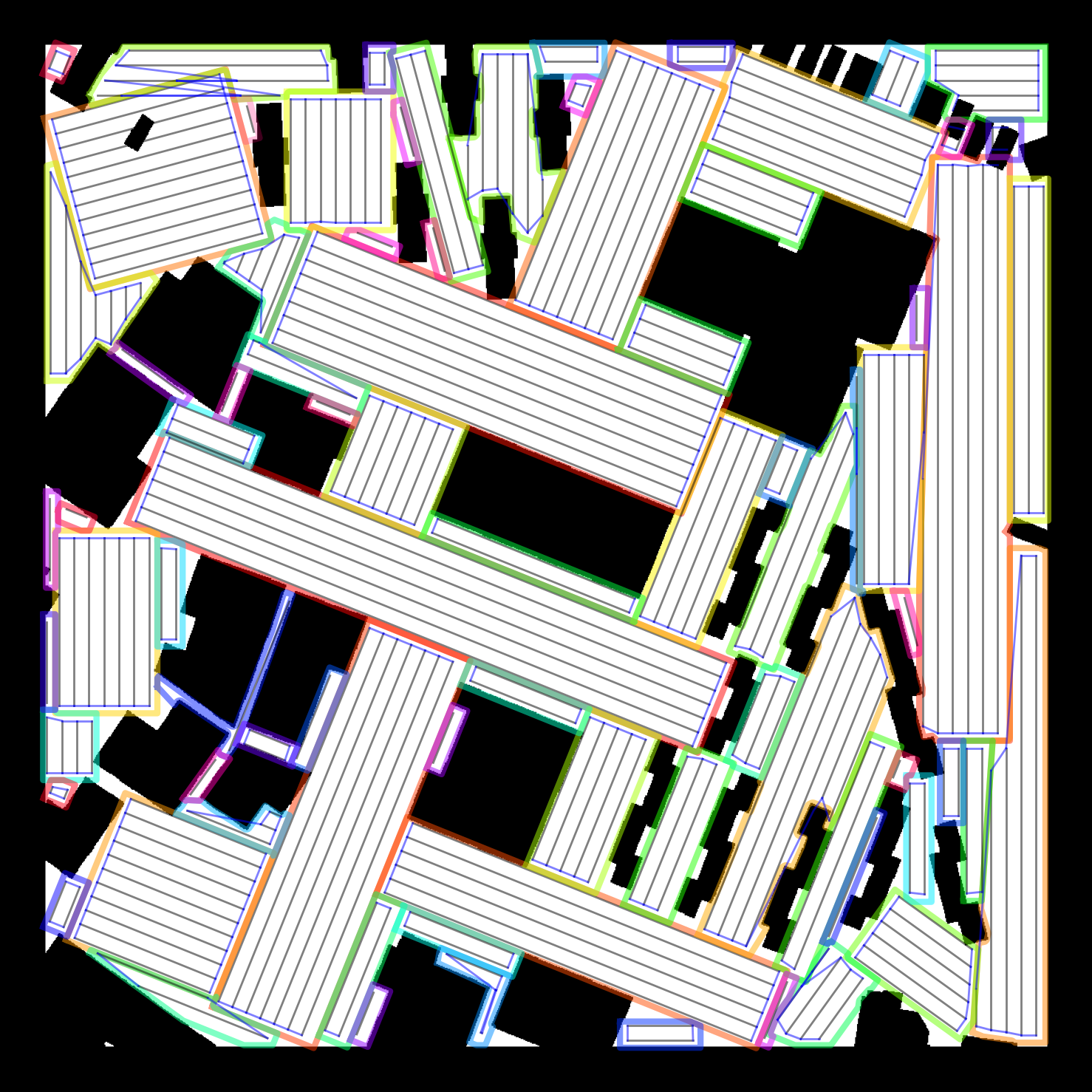}
    }
    \vspace{2mm}
    \subcaptionbox{Shanghai}
    {
    \includegraphics[width=0.17\linewidth]{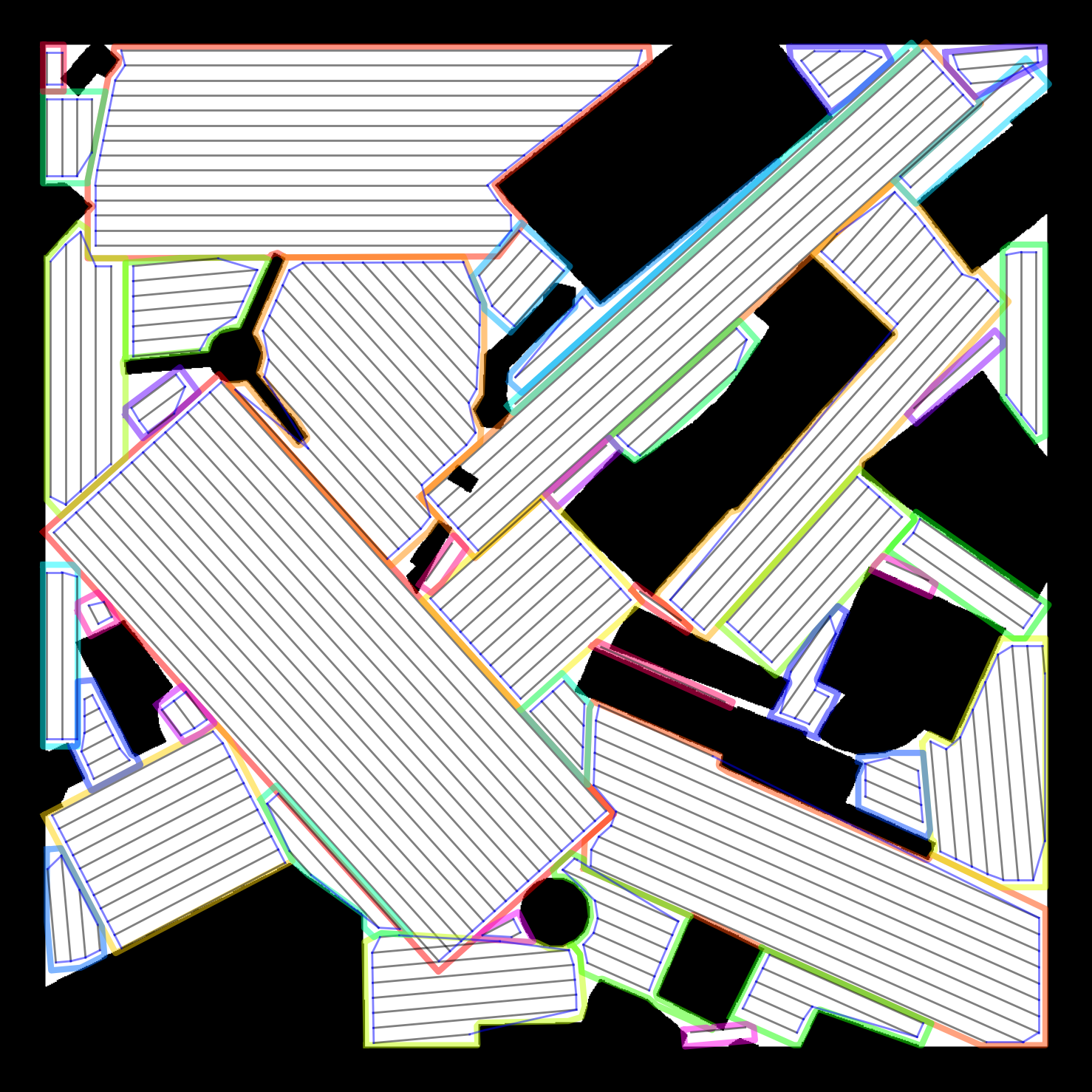}
    }
    \vspace{2mm}
    \subcaptionbox{Anonymous 1}
    {
    \includegraphics[width=0.28\linewidth]{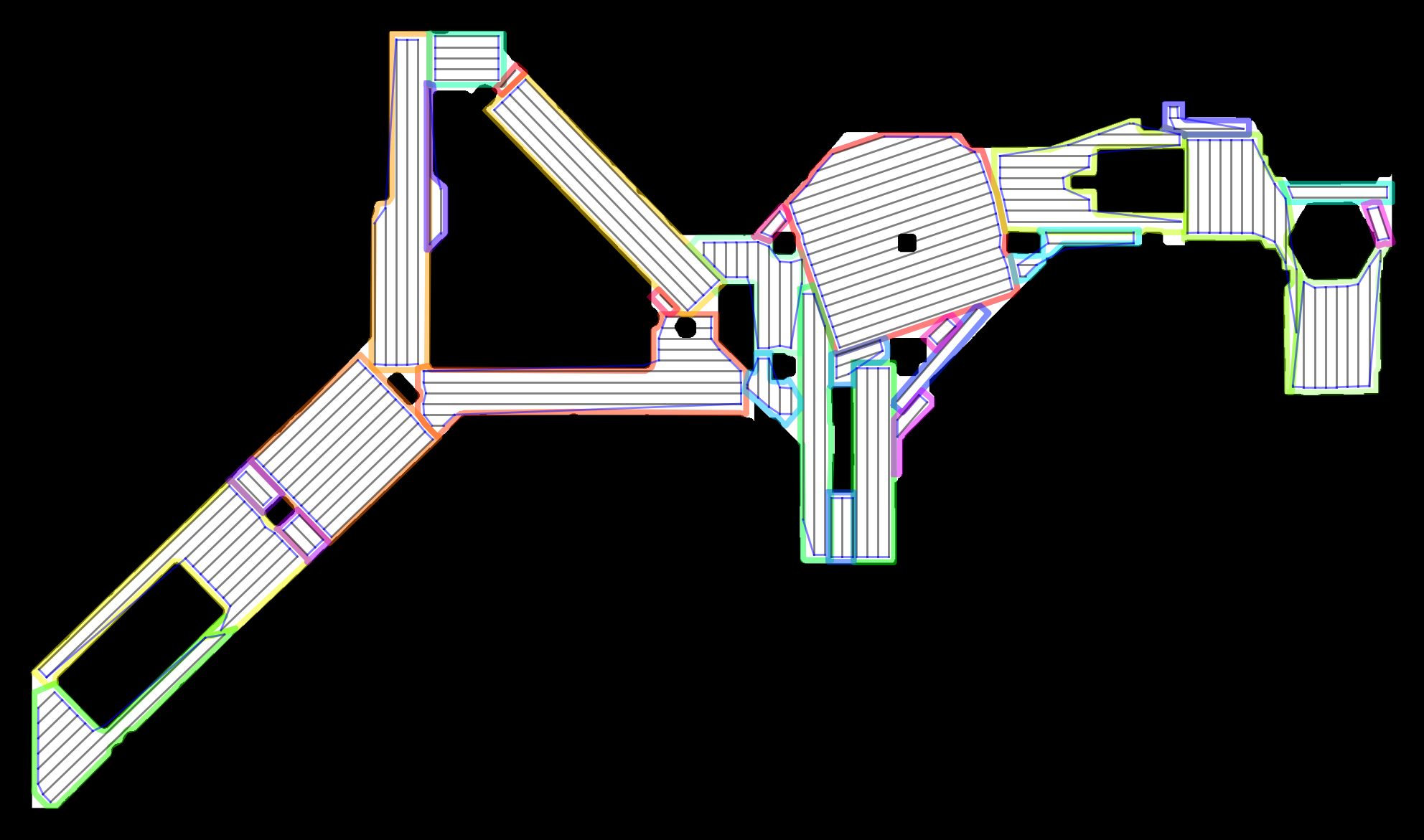}
    }
    \subcaptionbox{Anonymous 2}
    {
    \includegraphics[width=0.28\linewidth]{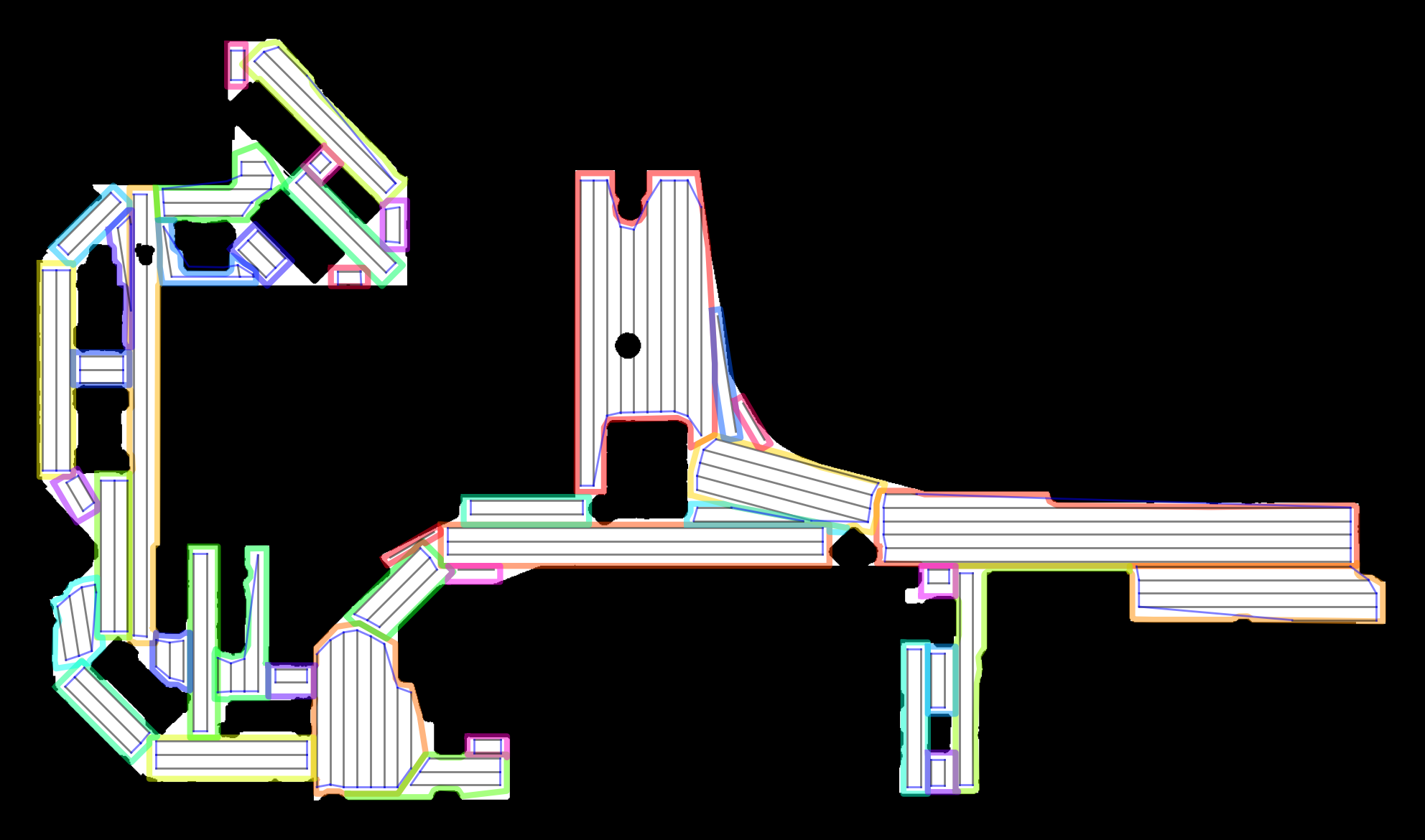}
    }
    \caption{\vspace{-2mm}Sector decomposition results for the test maps.}
    \label{fig:qualitative-results}
\end{figure*}


\section{Conclusion}

In this paper, we proposed an environment decomposition approach called \textit{sector decomposition} for robot coverage planning. The proposed approach aims to decompose the environment into the minimum number of rectangular sectors, which can each be covered using a lawnmower path oriented along its longest edge. We showed that sector decomposition is an instance of the submodular set cover (SSC) problem, and proposed a greedy approximation algorithm with a guarantee on the number of sectors in the decomposition. Finally, we present coverage planning results using the proposed approach for real-world environments and show that it improves upon existing coverage planning approaches \cite{bahnemannRevisitingBoustrophedonCoverage2021, rameshOptimalPartitioningNonConvex2022}. A future direction of this work is to minimize the number of turns in the coverage path generated using sector decomposition, for cases when robot turning is detrimental to coverage time and quality.


\section*{Acknowledgment}
This work was supported by Canadian Mitacs Accelerate Project IT16435 and Avidbots Corp, Kitchener, ON, Canada.

\bibliographystyle{IEEEtran}
\bibliography{references}

\begin{thebibliography}{10}
\providecommand{\url}[1]{#1}
\csname url@samestyle\endcsname
\providecommand{\newblock}{\relax}
\providecommand{\bibinfo}[2]{#2}
\providecommand{\BIBentrySTDinterwordspacing}{\spaceskip=0pt\relax}
\providecommand{\BIBentryALTinterwordstretchfactor}{4}
\providecommand{\BIBentryALTinterwordspacing}{\spaceskip=\fontdimen2\font plus
\BIBentryALTinterwordstretchfactor\fontdimen3\font minus \fontdimen4\font\relax}
\providecommand{\BIBforeignlanguage}[2]{{%
\expandafter\ifx\csname l@#1\endcsname\relax
\typeout{** WARNING: IEEEtran.bst: No hyphenation pattern has been}%
\typeout{** loaded for the language `#1'. Using the pattern for}%
\typeout{** the default language instead.}%
\else
\language=\csname l@#1\endcsname
\fi
#2}}
\providecommand{\BIBdecl}{\relax}
\BIBdecl

\bibitem{galceranSurveyCoveragePath2013}
E.~Galceran and M.~Carreras, ``A survey on coverage path planning for robotics,'' \emph{Robotics and Autonomous Systems}, vol.~61, no.~12, pp. 1258--1276, 2013.

\bibitem{palacios-gasosOptimalPathPlanning2017}
J.~M. {Palacios-Gas{\'o}s}, Z.~Talebpour, E.~Montijano, C.~Sag{\"u}{\'e}s, and A.~Martinoli, ``Optimal path planning and coverage control for multi-robot persistent coverage in environments with obstacles,'' in \emph{2017 {{IEEE International Conference}} on {{Robotics}} and {{Automation}} ({{ICRA}})}, May 2017, pp. 1321--1327.

\bibitem{jinCoverageControlAutonomous2013}
X.~Jin and A.~Ray, ``Coverage control of autonomous vehicles for oil spill cleaning in dynamic and uncertain environments,'' in \emph{2013 {{American Control Conference}}}, Jun. 2013, pp. 2594--2599.

\bibitem{grontvedDecentralizedMultiUAVTrajectory2023}
K.~A.~R. Gr{\o}ntved, U.~P.~S. Lundquist, and A.~L. Christensen, ``Decentralized {{Multi-UAV Trajectory Task Allocation}} in {{Search}} and {{Rescue Applications}},'' in \emph{2023 21st {{International Conference}} on {{Advanced Robotics}} ({{ICAR}})}.\hskip 1em plus 0.5em minus 0.4em\relax Abu Dhabi, United Arab Emirates: IEEE, Dec. 2023, pp. 35--41.

\bibitem{arkinApproximationAlgorithmsLawn2000}
E.~M. Arkin, S.~P. Fekete, and J.~S. Mitchell, ``Approximation algorithms for lawn mowing and milling,'' \emph{Computational Geometry}, vol.~17, no. 1-2, pp. 25--50, 2000.

\bibitem{bahnemannRevisitingBoustrophedonCoverage2021}
R.~B{\"a}hnemann, N.~Lawrance, J.~J. Chung, M.~Pantic, R.~Siegwart, and J.~Nieto, ``Revisiting {{Boustrophedon Coverage Path Planning}} as a {{Generalized Traveling Salesman Problem}},'' in \emph{Field and {{Service Robotics}}}, ser. Springer {{Proceedings}} in {{Advanced Robotics}}.\hskip 1em plus 0.5em minus 0.4em\relax Singapore: Springer, 2021, pp. 277--290.

\bibitem{huangOptimalLinesweepbasedDecompositions2001a}
W.~Huang, ``Optimal line-sweep-based decompositions for coverage algorithms,'' in \emph{2001 {{IEEE International Conference}} on {{Robotics}} and {{Automation}} ({{ICRA}})}, vol.~1, 2001, pp. 27--32 vol.1.

\bibitem{agarwalAreaCoverageMultiple2022a}
S.~Agarwal and S.~Akella, ``Area {{Coverage With Multiple Capacity-Constrained Robots}},'' \emph{IEEE Robotics and Automation Letters}, vol.~7, no.~2, pp. 3734--3741, Apr. 2022.

\bibitem{luTMSTCPathPlanning2023}
J.~Lu, B.~Zeng, J.~Tang, T.~L. Lam, and J.~Wen, ``{{TMSTC}}*: {{A Path Planning Algorithm}} for {{Minimizing Turns}} in {{Multi-Robot Coverage}},'' \emph{IEEE Robotics and Automation Letters}, vol.~8, no.~8, pp. 5275--5282, Aug. 2023.

\bibitem{ianenkoCoveragePathPlanning2020}
A.~Ianenko, A.~Artamonov, G.~Sarapulov, A.~Safaraleev, S.~Bogomolov, and D.-k. Noh, ``Coverage {{Path Planning}} with {{Proximal Policy Optimization}} in a {{Grid-based Environment}},'' in \emph{2020 59th {{IEEE Conference}} on {{Decision}} and {{Control}} ({{CDC}})}.\hskip 1em plus 0.5em minus 0.4em\relax Jeju, Korea (South): IEEE, Dec. 2020, pp. 4099--4104.

\bibitem{rameshOptimalPartitioningNonConvex2022}
M.~Ramesh, F.~Imeson, B.~Fidan, and S.~L. Smith, ``Optimal {{Partitioning}} of {{Non-Convex Environments}} for {{Minimum Turn Coverage Planning}},'' \emph{IEEE Robotics and Automation Letters}, vol.~7, no.~4, pp. 9731--9738, Oct. 2022.

\bibitem{rameshAnytimeReplanningRobot2023}
------, ``Anytime {{Replanning}} of {{Robot Coverage Paths}} for {{Partially Unknown Environments}},'' \emph{IEEE Transactions on Robotics (To appear)}, 2024, pre-print available at {http://arxiv.org/abs/2311.17837}.

\bibitem{marzehAlgorithmFindingLargest2019a}
Z.~Marzeh, M.~Tahmasbi, and N.~Mirehi, ``Algorithm for finding the largest inscribed rectangle in polygon,'' \emph{Journal of Algorithms and Computation}, vol.~51, no.~1, pp. 29--41, Jun. 2019.

\bibitem{danielsFindingLargestArea1997}
K.~Daniels, V.~Milenkovic, and D.~Roth, ``Finding the largest area axis-parallel rectangle in a polygon,'' \emph{Computational Geometry}, vol.~7, no. 1-2, pp. 125--148, Jan. 1997.

\bibitem{corahDistributedSubmodularMaximization2018}
M.~Corah and N.~Michael, ``Distributed {{Submodular Maximization}} on {{Partition Matroids}} for {{Planning}} on {{Large Sensor Networks}},'' in \emph{2018 {{IEEE Conference}} on {{Decision}} and {{Control}} ({{CDC}})}, Dec. 2018, pp. 6792--6799.

\bibitem{mehrSubmodularApproachOptimal2018}
N.~Mehr and R.~Horowitz, ``A {{Submodular Approach}} for {{Optimal Sensor Placement}} in {{Traffic Networks}},'' in \emph{2018 {{Annual American Control Conference}} ({{ACC}})}, Jun. 2018, pp. 6353--6358.

\bibitem{wolseyAnalysisGreedyAlgorithm1982}
L.~A. Wolsey, ``An analysis of the greedy algorithm for the submodular set covering problem,'' \emph{Combinatorica}, vol.~2, no.~4, pp. 385--393, Dec. 1982.

\bibitem{iyerSubmodularOptimizationSubmodular2013}
R.~K. Iyer and J.~A. Bilmes, ``Submodular {{Optimization}} with {{Submodular Cover}} and {{Submodular Knapsack Constraints}},'' in \emph{Advances in {{Neural Information Processing Systems}}}, vol.~26.\hskip 1em plus 0.5em minus 0.4em\relax Curran Associates, Inc., 2013.

\bibitem{NEURIPS2023_e5eaf67f}
W.~Chen and V.~Crawford, ``Bicriteria approximation algorithms for the submodular cover problem,'' in \emph{Advances in Neural Information Processing Systems}, vol.~36.\hskip 1em plus 0.5em minus 0.4em\relax Curran Associates, Inc., 2023, pp. 72\,705--72\,716.

\bibitem{changPolynomialSolutionPotatopeeling1986}
J.~S. Chang and C.~K. Yap, ``A polynomial solution for the potato-peeling problem,'' \emph{Discrete \& Computational Geometry}, vol.~1, no.~2, pp. 155--182, Jun. 1986.

\bibitem{hall-holtFindingLargeSticks2006a}
O.~{Hall-Holt}, M.~J. Katz, P.~Kumar, J.~S.~B. Mitchell, and A.~Sityon, ``Finding large sticks and potatoes in polygons,'' in \emph{Proceedings of the Seventeenth Annual {{ACM-SIAM}} Symposium on {{Discrete}} Algorithm - {{SODA}} '06}.\hskip 1em plus 0.5em minus 0.4em\relax Miami, Florida: ACM Press, 2006, pp. 474--483.

\bibitem{vandermeulenTurnminimizingMultirobotCoverage2019}
I.~Vandermeulen, R.~Gro{\ss}, and A.~Kolling, ``Turn-minimizing multirobot coverage,'' in \emph{2019 {{IEEE International Conference}} on {{Robotics}} and {{Automation}} ({{ICRA}})}, 2019, pp. 1014--1020.

\bibitem{obermeyerVisiLibityLibraryVisibility2008}
K.~J. Obermeyer and {Contributors}, ``{{VisiLibity}}: {{A C}}++ {{Library}} for {{Visibility Computations}} in {{Planar Polygonal Environments}},'' http://www.VisiLibity.org, 2008.

\bibitem{kalviainenProbabilisticNonprobabilisticHough1995}
H.~K{\"a}lvi{\"a}inen, P.~Hirvonen, L.~Xu, and E.~Oja, ``Probabilistic and non-probabilistic {{Hough}} transforms: Overview and comparisons,'' \emph{Image and Vision Computing}, vol.~13, no.~4, pp. 239--252, May 1995.

\end{thebibliography}

\end{document}